\documentclass[12pt]{colt2018}

\usepackage{times}

\def\E{{\mathbb E}}

\def\H{{\mathcal H}}
\def\I{{\mathcal I}}
\def\Q{{\mathcal Q}}
\def\X{{\mathcal X}}

\newcommand{\setc}{{[c]}}

\newcommand{\calq}{{\Q}}
\newcommand{\calqc}{{{\Q} \times [c]}}

\newcommand{\Qb}{{\bar{\Q}}}
\newcommand{\Bb}{{\bar{B}}}

\newcommand{\Gb}{{\bar{G}}}

\newcommand{\Rb}{{\bar{R}}}

\newcommand{\mub}{{\bar{\mu}}}

\newcommand{\hstar}{{h^*}}
\newcommand{\err}{\mbox{err}}
\newcommand{\errcomp}{{\mbox{err}_c}}
\newcommand{\quQ}{{q \in_\mu \Q}}
\newcommand{\eprime}{{\epsilon'}}

\newcommand{\advalg}{{stick-with-it algorithm}}
\newcommand{\Advalg}{{Stick-with-it algorithm}}
\newcommand{\basealg}{{base algorithm}}

\newcommand{\compindep}{{component-independent}}

\newcommand{\spolicy}{{smallest}}
\newcommand{\Spolicy}{{Smallest}}
\newcommand{\lpolicy}{{largest}}
\newcommand{\Lpolicy}{{Largest}}

\newtheorem{cor}[theorem]{Corollary}

\DeclareMathOperator*{\argmax}{arg\,max}

\title{Learning from partial correction}

\coltauthor{\Name{Sanjoy Dasgupta} \\ 
\addr University of California, San Diego \\
\AND
\Name{Michael Luby} \\ 
\addr Qualcomm Technologies, Inc.
}

\begin{document}
\date{}
\maketitle

\begin{abstract}
We introduce a new model of interactive learning in which an expert examines the predictions of a learner and partially fixes them if they are wrong. Although this kind of feedback is not i.i.d., we show statistical generalization bounds on the quality of the learned model.
\end{abstract}


\section{Introduction}

{\em Partial correction} is a natural paradigm for interactive learning. Suppose, for example, that a taxonomy is to be constructed on a large set of species $\I$, using steps of interaction with an expert. To see how one such step might go, let's say the learner's current model is some hierarchy $h$. Since $h$ is likely too large to be fathomed in its entirety, a small set of species $q \subset \I$ is chosen at random (for instance, $q = \{\mbox{\texttt{dolphin}, \texttt{elephant}, \texttt{mouse}, \texttt{rabbit}, \texttt{whale}, \texttt{zebra}}\}$), and the biologist is shown the restriction of $h$ to just these species, denoted $h(q)$. See Figure~\ref{fig:species}. If this subtree is correct, the biologist accepts it. If not, he or she provides a partial correction in the form of a triplet like $(\{\mbox{\texttt{dolphin}, \texttt{whale}}\}, \mbox{\texttt{zebra}})$, meaning ``there should be a cluster that contains {\texttt{dolphin}} and {\texttt{whale}} but not {\texttt{zebra}}'', that the correct tree must satisfy. This is easier than fixing the entire subtree.

Earlier models of interactive learning have typically adopted a {\em question-answer} paradigm: the learner asks a question and the expert answers it completely. In active learning of binary classifiers, for example, the question is a data point and the answer is a single bit, its label. When learning broader families of structures, however, partial correction can be more convenient and intuitive. In the tree case, the minimal question would consist of three species, and the expert would need to provide the restriction of the target hierarchy to these three leaves. But seeing a larger snapshot is helpful: it provides more context, and thus more guidance about the levels of granularity of clusters; it allows the expert to select one especially egregious flaw to fix, rather than having to correct minor mistakes that might in any case go away once the bigger problems are resolved; and, by allowing choice, it also potentially produces more reliable feedback. Finally, if the subtree is correct, the expert can accept it with a single click, and is saved the nuisance of having to enter it.

\begin{figure}
\begin{center}
\includegraphics[width=2in]{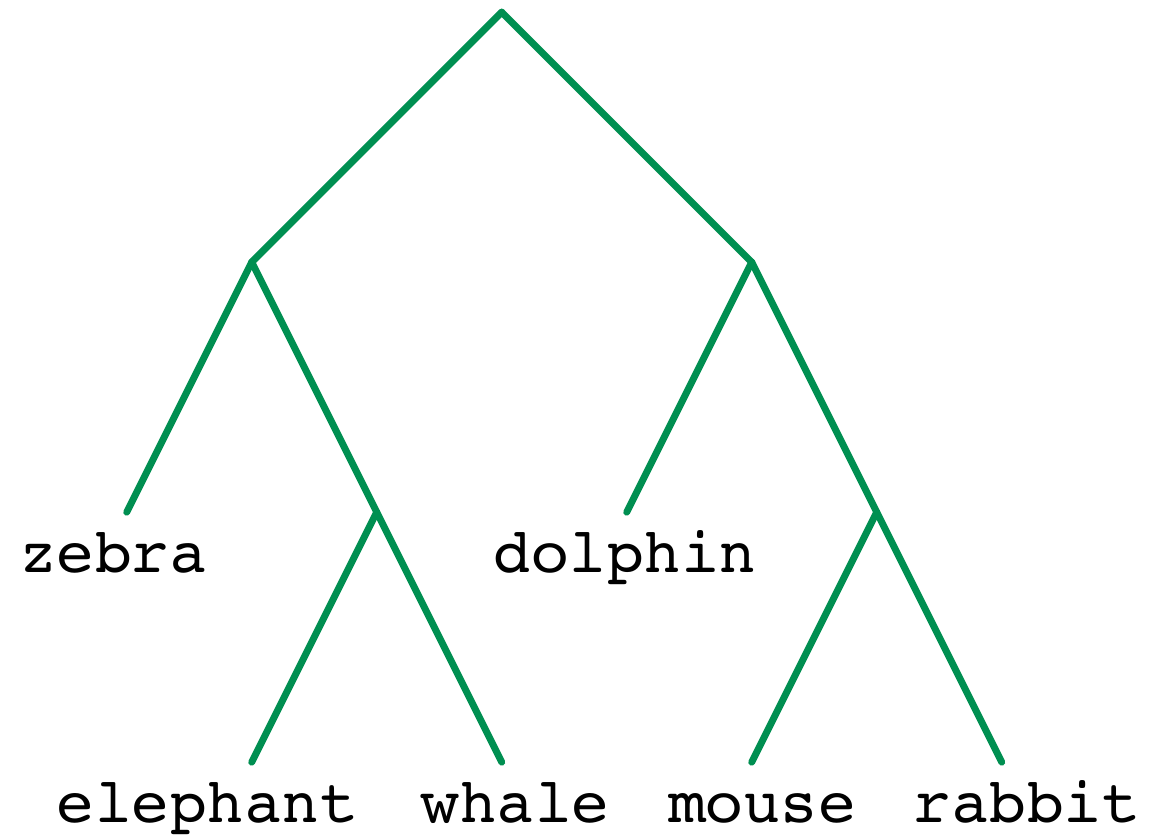}
\hskip.5in
\raisebox{.25in}{\includegraphics[width=1.25in]{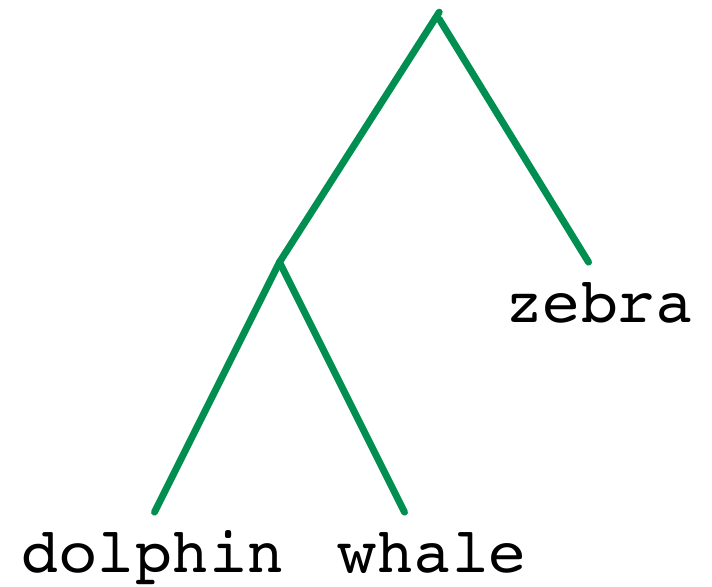}}
\end{center}
\caption{Left: A set $q$ of (say) six species is chosen at random, and the expert is shown $h(q)$, the restriction of the current hierarchy to these species. Right: The expert provides feedback of the form $h^*(x)$, where $x \subset q$ is some subset of three species on which $h$ is not correct, and $h^*$ is the target hierarchy.}
\label{fig:species}
\end{figure}

Formally, we assume that there is a space of structures $\H$ (for instance, trees over a fixed set of species), of which some $h^* \in \H$ is the {\em target}. Any $h \in \H$ can be specified by its answers to a set of questions $\Q$ (for instance, all subsets of six species). On each step of learning:
\begin{itemize}
\item The learner selects some hypothesis $h \in \H$ based on feedback received so far.
\item Some $q \in \Q$ is chosen at random.
\item The learner displays $q$ and $h(q)$ to an expert.
\item If $h(q)$ is correct, the expert accepts it. Otherwise, the expert fixes some part of it.
\end{itemize}
To formalize this partial correction, we assume that each $h(q)$ contains up to $c$ {\em atomic components}, individual pieces that can be corrected. In the tree example, these are triples of species, so $c = {6 \choose 3} = 20$. We index these components as $(q,1), \ldots, (q,c)$. The expert picks some $j$ for which $h(q,j) \neq h^*(q,j)$ and provides $h^*(q,j)$.

One case of technical interest, to which we will later return, is when the components of $q$ are independently chosen from the same distribution. We will call such a distribution on queries {\em \compindep}.

As another example, suppose each $q$ is a sequence of $c$ video frames of the driver's view in a car, and $h^*(q,j)$ is the appropriate driving action for the $j$th frame. On each step of interaction, a human labeler is shown $c$ frames, each labeled with an action, and either accepts all these actions as reasonable or corrects one of them.  
In this case it is unlikely that the distribution on queries is \compindep.

Formally, on each step of interaction, the learner either finds out that its prediction $h(q) = (h(q,1), \ldots, h(q,c))$ is entirely correct, or receives the correct value $h^*(q,j)$ for just one atomic component $j$. This kind of feedback is not i.i.d.: first, the feedback is constrained to be only one component on which $h$ is incorrect if there is such a component; and second, among possibly several such components, the expert chooses one in some arbitrary manner. Ideally, the expert's choices are illustrative and help the learning process, and we will soon see a simple example of this kind. But in this paper we also study the other extreme: is it true that even if the expert adversarially chooses what feedback to give, the same rate of convergence as i.i.d. sampling is always assured? We show that this is indeed the case, and this is a crucial sanity check for the partial correction model.  Furthermore, we show that our algorithms are optimal with respect to natural metrics.

\subsection{Learning procedure}
\label{sec:alg}

Let $\mu$ be a probability distribution on $\Q$, and let $\quQ$ indicate that $q$ is chosen 
independently from $\Q$ according to $\mu$; in the tree example above, $\Q$ 
is all subsets of six species and $\mu$ is the uniform distribution on $\Q$. 
On step $t = 1, 2, \ldots$ of learning,
\begin{enumerate}
\item Learner selects some $h_t \in \H$ consistent with all feedback received so far
\item Choose $\quQ$, where $q$ has $c$ atomic components, $(q,1), \ldots, (q,c)$.
\item Learner displays $q$ and $h_t(q)$ to expert
\item If $h_t(q)$ is correct:
\begin{itemize}
\item Expert feeds back that $h_t(q)$ is correct 
\item Feedback implicitly provides, for all $j \in [c]$, $h^*(q,j) = h_t(q,j)$.
\end{itemize}
Else $h_t(q)$ is incorrect:
\begin{itemize}
\item Expert chooses $1 \leq j \leq c$ for which $h_t(q,j) \neq h^*(q,j)$
\item Expert feeds back $j$ and $h^*(q,j)$.
\end{itemize}
\end{enumerate}

\subsection{Results}

The error of a hypothesis $h \in \H$ can be measured in two ways: in terms of full questions $q \in \Q$,
\begin{equation*} \err(h) = \mbox{Pr}_{\quQ}[h(q) \neq h^*(q)] .\end{equation*}
or in terms of atomic components $(q,j)$:
\begin{equation*} \errcomp(h) = \mbox{Pr}_{\quQ, j \in_R [c]} [h(q,j) \neq h^*(q,j)] .\end{equation*}
These are related by 
$\errcomp(h) \leq \err(h) \leq c \cdot \errcomp(h)$.
Note that $\err(h) \approx c \cdot \errcomp(h)$ if $\mu$ is \compindep\ and $\err(h)$ is small.

An important complexity metric is the expert cost per step to provide feedback.  
This cost can be substantially lower in the new model: The expert can choose a 
component that is easiest to determine is incorrect amongst a set of $c$ components, 
instead of being required to provide feedback for a particular component.
We leave to future work the study of this metric in more detail.

Another crucial complexity metric is the number of steps of feedback required to learn.
We start with a simple one-dimensional example (Section~\ref{sec:lower-bound}) that illustrates how the expert's choice of feedback can significantly affect this metric. In the example, one feedback strategy reduces the number of steps needed for learning by a factor of up to $c$ (so that each feedback component is about as valuable as $c$ randomly chosen components), while a different strategy increases the number of steps by a factor of $\Omega(c)$ (slows down learning).  

The example demonstrates that the number of steps needed to learn can vary by wide margins depending on the expert policy.
Our main results (Theorem~\ref{main theorem}, and the more general Theorem~\ref{stick-with-it-theorem}) show that, despite this, there is a reasonable bound on the number of steps to learn no matter how adversarial the expert policy: For any expert policy, for any $0 < \delta, \epsilon < 1$, with probability $1-\delta$ the \basealg\ of Section~\ref{sec:alg} produces a hypothesis $h$ with $\err(h) \leq \epsilon$ within $O((c/\epsilon) \cdot \log (|\H|/\delta))$ steps of feedback. 
Moreover (Theorem~\ref{thm:generalization-errc}), with probability $1-\delta$, after the same number of steps, {\em all} consistent hypotheses have $\errcomp(h) \leq \epsilon$.
Section~\ref{sec lower} shows that this number of steps is needed for at least some examples.  

In the standard supervised learning model, labeled data is provided in advance, after which a consistent hypothesis is sought. In our protocol, feedback is obtained in steps, and the learner needs to maintain a consistent hypothesis throughout the process. Because it can be expensive to continually select a consistent hypothesis, we introduce the {\em \advalg}, a variant of the \basealg, that might be preferable in practice (Section~\ref{sec:stick-with-it}). Rather than always having to select a hypothesis that is consistent with all feedback received so far at each step, it only updates its hypothesis $O(c)$ times during the entire learning process. 

To obtain these sample complexity bounds, we look at the effective distribution $w_t$ over atomic components $(q,j)$ at each time step $t$, which is a function of previous feedback, the learning algorithm's choice of current model $h_t$, and the expert's criterion for selecting what to correct. This can be quite different from the distribution that would be easy to analyze, where $\quQ$ and $j$ is chosen at random; in particular, $w_t$ can be zero at many $(q,j)$ with $\mu(q) > 0$. Nonetheless, we show that over time, no matter what policy the expert chooses, $w_t$ cannot avoid covering the whole $\Q \times [c]$ space in some suitably amortized sense.

\subsection{Related work}

The growing area of {\em interactive learning} raises many new problems and challenges. Here we have formalized an interactive protocol that is quite natural and intuitive in terms of human-computer interface, but breaks the statistical assumptions that underlie generalization results in other settings like the PAC model of~\citet{V84}. Our key technical contribution is to establish sample complexity bounds in this novel framework.

Most work in interactive learning has employed question-and-answer protocols, in which the learner asks for a specific piece of information, like the label of a point, and gets back the full answer. This is, for instance, the typical setting for active learning of classifiers~\citep{S12}.

One previously-studied model that uses partial correction is {\em learning from equivalence queries} \citep{A88}. In that setting, each round of learning proceeds as follows:
\begin{itemize}
\item the learner suggests a concept
\item the teacher either accepts it, or provides a counterexample
\end{itemize}
Early work focused on Boolean concept classes like disjunctions, while more recently this model has been extended to broader families of models, such as clustering~\citep{BB08,ABV17,EK17}. 

One general issue with the equivalence query model is that the learner is expected to provide the entire concept at each round; this may in general be very large (a clustering of a million points, for instance) or hard to understand (a neural net, say). In our model, on the other hand, the learner only provides a small constant-sized snapshot of the concept on each round, in a readily-understandable form. Because this snapshot is chosen at random, we are faced with a statistical challenge that is entirely absent from the equivalence query model, and our paper is devoted to addressing this technical problem.

\section{An illustrative example}
\label{sec:lower-bound}

Suppose $\X = [0,1]$ and the goal is to learn a threshold classifier:
\begin{equation*} \H = \{h_v: v \in [0,1]\}, \ \ h_v(x) = 1(x > v) .\end{equation*}
Say the target threshold is $0$ (that is, $h^* = h_0$), so that the correct label for all points in $(0,1]$ is 1. If we were learning from random examples $(x, h^*(x))$ then, no matter the distribution on $\X$, after $O(1/\epsilon)$ samples, with probability close to one, all consistent hypotheses $h$ would have $\err(h) \leq \epsilon$. Thus, after $O(1)$ instances, the error would be lower than any pre-specified constant.

\subsection{Uniformly distributed, component-independent queries}

We will consider queries consisting of $c$ points from $\X$; that is, $\Q = \X^c = [0,1]^c$, where we define $h_v(x_1, \ldots, x_c) = (h_v(x_1), \ldots, h_v(x_c))$. Let $\mu$ be the uniform distribution over $\Q$. 
Since the target threshold is 0, the probability that $h_v$ errs on a single component is $\errcomp(h_v) = v$, while the probability that it errs on a query consisting of $c$ components is $\err(h_v) = 1-(1- v)^c$, for any $v \in [0,1]$. Thus $\err(h) \approx c \cdot\errcomp(h)$ if $\err(h)$ is small. 

On each round of interaction, the expert is shown $c$ points in $\X$, along with proposed labels, and provides feedback on at least one of these points. After $t$ such steps, let $v_t$ denote the smallest-valued point in $\X$ on which the expert has provided feedback. Thus, the version space at time $t$ consists exactly of classifiers $h_v$ with threshold $v \leq v_t$. We'll try to understand how the rate of convergence of $v_t$ to zero is affected by $c$ and by the expert labeler's policy for which errors to correct. For simplicity, we will take the learner's hypothesis at time $t$ to be $h_{v_t}$.

Each query consists of $x_1, \ldots, x_c$ chosen uniformly at random from $[0,1]$, and labeled according to $h_{v_t}$. We consider two expert policies:
\begin{itemize}
\item ``\Lpolicy'': the expert picks the largest-valued $x_i$ whose label is incorrect. This corresponds to a natural tendency to fix the biggest mistake, but is the least informative correction.
\item ``\Spolicy'': the expert picks the smallest-valued $x_i$ whose label is incorrect. This is the most informative correction.
\end{itemize}
Based on this feedback, let random variable $V_{t+1}$ denote the learner's updated threshold. What is the expected value of $V_{t+1}$?

{\em When the labeling policy is ``\lpolicy''}: For any $v \in [0,v_t)$, the only way $V_{t+1}$ can exceed $v$ is if either all the $x_i$ are $\geq v_t$ (and are thus correctly labeled by $h_{v_t}$) or if at least one of the $x_i$ lies in $(v,v_t)$ (in which case, there is at least one error, but the largest component in error exceeds $v$):
\begin{align*}
\Pr(V_{t+1} > v \ |\  V_t = v_t)
&=
\Pr(\mbox{all $x_i \geq v_t$}) + (1 - \Pr(\mbox{no $x_i$ in $(v,v_t)$})) \\
&=
(1-v_t)^c + (1 - (1-(v_t-v))^c) 
\end{align*}
Therefore, by calculation,
\begin{equation*}
\E[V_{t+1} \ |\  V_t = v_t] 
= 
\int_{0}^{v_t} \Pr(V_{t+1} > v \ | \ V_t = v_t) dv
= 
v_t - \frac{1 - (1-v_t)^{c} \cdot (1 + c \cdot v_t)}{c+1} .
\end{equation*}

{\em When the labeling policy is ``\spolicy''}: For $v \in [0,v_t]$,  the only way $V_{t+1}$ can exceed $v$ is if none of the $x_i$ lie in $[0,v]$, so $\Pr(V_{t+1} > v \ |\  V_t = v_t) = (1-v)^c$,
whereupon, by a similar integral,
\begin{equation*} \E[V_{t+1} \ | \ V_t = v_t] 
\ = \ 
\frac{1 - (1-v_t)^{c+1}}{c+1} .
\end{equation*}

When $c=1$, the two policies coincide and $\E[V_{t+1} | v_t] = v_t - v_t^2/2$, so the {\em expected instantaneous reduction} in $V_t$, that is $\E[v_t - V_{t+1}]$, from seeing a single-point query is $v_t^2/2$. How does this compare to the expected instantaneous reduction from queries consisting of $c$ points? The {\em ratio} of the expected reduction with $c$-point queries to the expected reduction with $1$-point queries is shown in Figure~\ref{fig:ratios} for $c=4,8$ and for the ``\spolicy'', ``\lpolicy'' expert policies. The ratio is given at each value $v_t$.

\begin{figure}
\begin{center}
\includegraphics[width=2.75in]{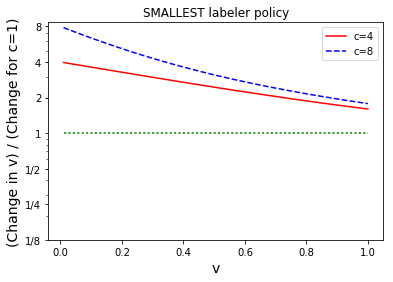}
\hskip.25in
\includegraphics[width=2.75in]{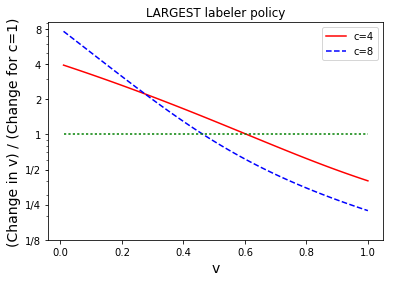}
\end{center}
\caption{Ratio between expected reduction in error from a query consisting of $c$ points versus a single-point query, for two expert policies (``\spolicy'' and ``\lpolicy'') and values $c=4,8$.}
\label{fig:ratios}
\end{figure}

As expected, under the ``\spolicy'' labeling policy, $c$-point queries are always more helpful than single-point queries. Under the ``\lpolicy'' policy, this is true only when $v_t$ is sufficiently small. In either case, when $v_t$ gets close to zero, the single label yielded by a $c$-point query is roughly as informative as $c$ random labeled points. This can be checked directly from the expressions above.

This example shows that the rate of convergence of learning by partial correction depends on the labeler's choice of which errors to fix. Even in this simple setting, different labeler policies can speed up or slow down convergence by factors up to $c$. We now formalize lower bounds of this type.

\subsection{A lower bound on component-level error}

We continue with the one-dimensional example, with the same hypothesis class and the same target, but we now turn to distributions that are not component-independent.

As before, we will consider a learner that begins with a threshold of 1, and at any given time, chooses the largest threshold consistent with all feedback so far: namely, the smallest-valued point for which it has received feedback.

\subsubsection{A single query, repeated}

To start with an especially simple case, say the distribution $\mu$ over $\Q$ is supported on a single point, $(1/c, 2/c, \ldots, 1)$. Suppose moreover that the expert labeler behaves as follows: when presented with a labeling of the points $1/c, 2/c, \ldots, 1$, he/she always chooses to ``correct the most glaring flaw'', that is, the highest value for which a 0 label is suggested.

It is clear that $x = 1$ is labeled in the first round, $x = (c-1)/c$ in the second round, $x = (c-2)/c$ in the third round, and so on. The labeler's behavior is hardly pathological. And yet, it takes $c/2$ rounds of interaction to bring the error down to $1/2$. If the feedback were on random components, then $O(1)$ rounds would have been sufficient.

\subsubsection{Lower bound}

Pick any $\epsilon > 0$, and now consider a distribution $\mu$ over $\Q$ that is supported on just two points:
\begin{equation*}
\begin{array}{cc}
\left( \frac{1}{2c}, \frac{2}{2c}, \ldots, \frac{1}{2}\right) & \mbox{probability $2\epsilon$} \\ [.25em]
\left( \frac{1}{2} + \frac{1}{2c}, \frac{1}{2} + \frac{2}{2c}, \ldots, 1 \right) & \mbox{probability $1 - 2\epsilon$} 
\end{array}
\end{equation*}
Any hypothesis with $\mbox{err}_c(h_v) \leq \epsilon$ must have $v \leq 1/4$. In order to achieve this, the learner must see the first point at least $c/2$ times, which requires seeing $\Omega(c/\epsilon)$ samples overall, with high probability.

We have established the following.
\begin{theorem}
There is a concept class $\H$ of VC dimension 1 such that for any $\epsilon > 0$, it is necessary to have $\Omega(c/\epsilon)$ rounds of feedback in order to be able to guarantee that with high probability, all hypotheses $h$ consistent with this feedback have $\mbox{err}_c(h) \leq \epsilon$.
\label{thm:lower-bound-errc}
\end{theorem}

\section{Main result}

For each $h \in \H$, let 
\begin{eqnarray*}
B(h) & = & \{q \in \calq: h \mbox{ is incorrect on } q\}, \\
G(h) & = & \{q \in \calq: h \mbox{ is correct on } q\}
\end{eqnarray*}
Note that $\err(h) = \mu(B(h))$ is the probability that $h$ is incorrect on a randomly chosen query.
We say that hypothesis $h$ is $(1-\epsilon)$-good if $\mu(B(h)) \le\epsilon$.
On input $(\epsilon,\delta)$, the goal is to find an $h \in \H$
that is $(1-\epsilon)$-good with probability at least $1-\delta$.

\begin{theorem}
\label{main theorem}
Let $\ell = \log(|\H|/\delta)$, let $\eprime = \epsilon/2$, and
let $N = c \cdot \left(\frac{\ell}{\eprime}+1 \right).$
The \basealg\ of Section~\ref{sec:alg} produces an $(1-\epsilon)$-good hypothesis within $2 \cdot N$
steps with probability at least $1-\delta$.
\end{theorem}

It is interesting to compare Theorem~\ref{main theorem} to standard generalization bounds in the case when $\mu$ is \compindep.
Theorem~\ref{main theorem} shows that after at most $2 \cdot N = O(c \cdot \log(|\H|)/\epsilon)$ steps 
the output hypothesis $h$ satisfies $\err(h) \le \epsilon$, which implies (roughly) that $\errcomp(h) \le \epsilon/c$ if $\mu$ is \compindep.
Under standard bounds, this is the same number of steps that would be needed to achieve component error $\epsilon/c$ 
when each question is a single component and the expert provides complete feedback for each question.
Of course, the bound of Theorem~\ref{main theorem} applies whether or not $\mu$ is \compindep.

The remainder of this section concentrates on proving Theorem~\ref{main theorem}.
The analysis procedes in two phases: the first phase considers the first $N$ steps,
and the second phase considers the subsequent $N$ steps. 
Writing $[c]$ for $\{1,2,\ldots, c\}$, let
\begin{eqnarray*}
\Qb & = & \calqc,  \\
\Bb(h) & = & \{ (q,j) \in \Qb: q \in B(h) \mbox{ and } h(q,j) \not= \hstar(q,j) \}, \\
\Gb(h) & = &  G(h) \times \setc.
\end{eqnarray*}

\subsection{Effective sampling distribution}

Let $h_t$ be the current hypothesis at the beginning of step $t$. The feedback at time $t$ will depend on $h_t$, on the query $q$ (chosen from distribution 
$\mu$), and on the expert's choices. For purposes of the analysis, we define the {\em effective sampling distribution} $w_t$ over $\Qb = \calqc$, as follows:
\begin{itemize}
\item For all $(q,j) \in \Bb(h_t)$, let $\gamma(q,j)$ denote the conditional probability 
that the expert provides feedback on $(q,j)$ when query $q$ is made. Define
$w_t(q,j) = \mu(q) \cdot \gamma(q,j)$.
\item For all $q \in G(h_t)$ calculate $w_t(q,1), \ldots, w_t(q,c)$, summing to $\mu(q)$, 
as specified below in Lemma~\ref{good wt lemma}. 
\end{itemize}
Finally, let \begin{equation*}W_t(q,j) = w_1(q,j) + \cdots + w_t(q,j)\end{equation*} 
denote the sum of the individual distributions up to step $t$.
Note that at each step $t$, for each $q \in \calq$, 
we have $w_t(q,\setc) = w_t(q, 1) + \cdots + w_t(q,c) = \mu(q)$ and thus $W_t(q,\setc) = t \cdot \mu(q)$. 

\begin{lemma}
\label{good wt lemma}
For all $q \in G(h_t)$, non-negative values for $w_t(q,1), \ldots, w_t(q,c)$, summing to $\mu(q)$,
can be calculated such that the following property holds: for any $j$ with $w_t(q,j) > 0$,
\begin{equation*}W_t(q,j) = W_{t-1}(q,j) + w_t(q,j) \leq \frac{t \cdot \mu(q)}{c}.\end{equation*}
\end{lemma}
\begin{proof}
We begin with some intuition. For all $q \in G(h_t)$, we want to add a total of $\mu(q)$ to the entries $W_{t-1}(q,j)$ for $j \in [c]$. We will only add to entries that are $\leq t \cdot \mu(q)/c$, and we will make sure not to exceed this threshold. We can do this because $W_{t-1}(q,[c]) = (t-1) \cdot \mu(q)$, and thus the average $W_{t-1}(q,j)$ is $(t-1) \cdot \mu(q) / c$.

Formally, we can choose $w_t(q,1), \ldots, w_t(q,c)$ as follows.
Let $j_1,\ldots,j_c$ be an ordering of the elements of $\setc$ such that 
\begin{equation*}W_{t-1}(q,j_1) \le W_{t-1}(q,j_2) \le \cdots \le W_{t-1}(q,j_c).
\end{equation*}
Let
\[ i^* = \argmax_i \{ W_{t-1}(q,j_i) \le  t \cdot \mu(q)/c \},\] 
\[ \Delta^+ = \sum_{i \le i^*} t \cdot \mu(q)/c-W_{t-1}(q,j_i), \]
and 
\[\Delta^- = \sum_{i > i^*} t \cdot \mu(q)/c-W_{t-1}(q,j_i).\]
From the above it follows that $\Delta^+ + \Delta^- =  \mu(q)$ and $\Delta^- \le 0$, and thus $\Delta^+ \ge \mu(q)$.
This ensures that if we start by ``filling up'' entry $j_1$ to threshold $t \cdot \mu(q)/c$, then entry $j_2$, and so on, then we can fill up a total of $\mu(q)$ without any entry exceeding $t \cdot \mu(q)/c$, as described in the following algorithm:

Initialize $\Delta = \mu(q)$, $w_t(q,j_1) = \cdots = w_t(q,j_c) = 0$.

Repeat the following for $i=1,\ldots,c$ until $\Delta = 0$: 

\hspace{0.5in}
Reset $w_t(q,j_i) = \min\left\{\frac{t\cdot \mu(q)}{c} - 
W_{t-1}(q,j_i), \Delta \right\}$ 

\hspace{0.5in}
Reset $\Delta = \Delta - w_t(q,j_i)$.
\end{proof}

\subsection{Eliminating inconsistent hypotheses}

Next, we use a large deviation argument to assert that any suboptimal hypothesis $h$ will be eliminated once the region in which it is incorrect, $\Bb(h)$, has been sufficiently sampled. In what follows, recall that $\ell = \log (|\H|/\delta)$.
\begin{lemma}
With probability at least $1-\delta$, the following holds for all $h \in \H$: if
there is a step $t$ at which $W_t(\Bb(h)) \geq \ell$, then $h$ is not consistent with
the feedback received by the end of that step.
\label{lemma:generalization}
\end{lemma}
\begin{proof}
Pick any $h \in \H$. It is eliminated if feedback is received on any $(q,j) \in \Bb(h)$.
The probability that this happens at step $t$ is at least $w_t(\Bb(h))$.

Let $t$ be the first step at which $W_t(\Bb(h)) \geq \ell$. The probability that $h$ is not
eliminated by the end of step $t$ is at most
\begin{equation*} (1-w_1(\Bb(h)))\cdot (1-w_2(\Bb(h))) \cdots (1-w_t(\Bb(h))) \leq \exp(-W_t(\Bb(h))) \leq \exp(-\ell) = \frac{\delta}{|\H|}.\end{equation*}
Taking a union bound over $\H$, with probability at least $1-\delta$, any hypothesis $h$ is eliminated from the version space by the step at which $W_t(\Bb(h)) \geq \ell$.
\end{proof}

We hereafter assume $W_{t-1}(\Bb(h_t)) < \ell$ if $h_t$ is selected as the 
current hypothesis at the beginning of step $t$.

\subsection{\Advalg}
\label{sec:stick-with-it}

There are some practical issues with the \basealg\ of Section~\ref{sec:alg}.  
One issue is that at the beginning of every step,
a hypothesis needs to be selected that is consistent with all feedback so far.
A second issue is that a separate procedure is needed to evaluate whether a given hypothesis is $(1-\epsilon)$-good,
in order to terminate the \basealg\ with a hypothesis that is verified to be $(1-\epsilon)$-good.


We introduce the {\em \advalg}, a generalization of the \basealg,
that addresses these issues.
We use an integer $k \ge 1$ to describe the following simple change to the \basealg:
Instead of selecting a current hypothesis at the beginning of each time step (that is consistent with all feedback received), a current hypothesis is selected each $k$ steps. 
Once selected, it is used as the current hypothesis for the next $k$ consecutive steps,
even if it becomes inconsistent with feedback received during these $k$ steps.  
(This is where ``stick-with-it'' comes from.) 

\begin{theorem}
\label{stick-with-it-theorem}
Let $\ell = \log(|\H|/\delta)$, let $\eprime = \epsilon/2$, and
let $N = c \cdot \left(\frac{\ell}{\eprime}+k \right).$
The \advalg\ produces an $(1-\epsilon)$-good hypothesis within $2 \cdot N$
steps with probability at least $1-\delta$.
\end{theorem}

Subsections~\ref{phase 1 sec} and~\ref{phase 2 sec} 
below provide the proof of Theorem~\ref{stick-with-it-theorem}, 
which immediately also proves Theorem~\ref{main theorem} (taking $k=1$). 
Setting \[k =\frac{\ell}{\eprime} = \frac{2 \cdot \ell}{\epsilon} \] 
results in a \advalg\ with the following properties:
\begin{itemize}
\item
The total number of steps is at most $2 \cdot N \le \frac{8 \cdot c \cdot \ell}{\epsilon}.$
\item
A new current hypothesis is selected at most 
$\frac{2 \cdot N}{k} \le 4 \cdot c$ times, and thus there are at most $4 \cdot c$ different current hypotheses. 
\item
A new current hypothesis remains the current hypothesis for enough steps 
to determine if it is $(1-\epsilon)$-good, and if it is $(1-\epsilon)$-good 
then the \advalg\ terminates. 
\end{itemize}
The \advalg\ is close-to-optimal in the following metrics (see Section~\ref{sec lower}):
\begin{itemize}
\item
The bound on the number of steps, including steps to verify that 
the output hypothesis is $(1-\epsilon)$-good
\item
The bound on the number of times the current hypothesis needs to be updated
\end{itemize}

\subsection{Analysis for Phase 1}
\label{phase 1 sec}

Consider a first phase consisting of the first $N$ steps. Let $\tau$ be a threshold value. 
We will think of an atomic question $(q,j)$ as having been adequately sampled when $W_t(q,j)$ reaches $\tau \cdot \mu(q)$. Define
$$ \Rb_t = \{(q,j): W_t(q,j) > \tau \cdot \mu(q) \}$$
to be the set of $(q,j)$ that have been oversampled by the end of time step $t$. We will see that for a suitable setting of $\tau$, the effective sampling distribution $w_t$ at time $t$ places little weight on $\Rb_t$. To show this, we partition $\Rb_t$ into $\Gb(h_t) \cap \Rb_t$ and $\Bb(h_t) \cap \Rb_t$.

\begin{lemma}
If $c \cdot \tau \geq N$ then $w_t(\Gb(h_t) \cap \Rb_t) = 0$ for any $t \leq N$.
\label{lemma:good-set-not-oversampled}
\end{lemma}
\begin{proof}
Pick any $(q,j) \in \Gb(h_t)$. If $w_t(q,j) > 0$ then we have from Lemma~\ref{good wt lemma} that
$$ W_t(q,j) \leq \frac{t}{c} \cdot \mu(q) \leq \frac{N}{c} \cdot \mu(q) \leq \tau \cdot \mu(q).$$
Thus any such $(q,j)$ is not in $\Rb_t$.
\end{proof}

\begin{lemma}
At any time $t$, if current hypothesis $h_t$ was selected within the previous $k$ steps then 
\[ w_t(\Bb(h_t) \cap \Rb_t) \leq \frac{\ell}{\tau - k}.\]
\label{lemma:bad-set-not-oversampled}
\end{lemma}
\begin{proof}
For any $(q,j) \in \Rb_t$, we have 
$$ W_{t-k}(q,j) \geq W_t(q,j) - k \cdot \mu(q) > (\tau - k) \cdot \mu(q).$$
Thus
\[ w_t(\Bb(h_t) \cap \Rb_t) 
= \sum_{(q,j) \in \Bb(h_t) \cap \Rb_t} w_t(q,j)
\leq \sum_{(q,j) \in \Bb(h_t) \cap \Rb_t} \mu(q)
< \frac{1}{\tau-k} \cdot W_{t-k}(\Bb(h_t))
< \frac{\ell}{\tau -k},
\]
where the last inequality is because $W_{t-k}(\Bb(h_t)) < \ell$ from Lemma~\ref{lemma:generalization} when $h_t$ is selected.
\end{proof}

With Lemmas~\ref{lemma:good-set-not-oversampled} and \ref{lemma:bad-set-not-oversampled} in mind, we set
\begin{equation*}\tau = \frac{N}{c} = \frac{\ell}{\eprime} + k\end{equation*}
whereupon the following is immediate.
\begin{lemma}
\label{lemma:overall_benefit2}
At any step $t \le N$, $w_t(\Rb_t) \le \eprime$.
\end{lemma}

Let $\widehat{W}_t(q,j) = \min\{W_t(q,j), \tau \cdot \mu(q)\}$. 
Summing over all $(q,j)$, we have $\widehat{W}_t(\Qb) \leq N$.
\begin{cor}
\label{corr:overall_benefit2}
$\widehat{W}_N(\Qb) \ge (1-\eprime) \cdot N$.
\end{cor}
\begin{proof}
An immediate consequence of Lemma~\ref{lemma:overall_benefit2}.
\end{proof}

\subsection{Analysis for Phase 2}
\label{phase 2 sec}

We now finish the proof of 
Theorem~\ref{stick-with-it-theorem}. 
\begin{proof}
Consider a second phase of $N$ additional steps.
Let $h_t$ be the current hypothesis for one of these steps.
If $\mu(B(h_t)) \ge 2 \cdot \eprime$ then $\mu(B(h_t)) - \eprime \ge \eprime$,
and Lemma~\ref{lemma:bad-set-not-oversampled} 
implies that $w_t(\Bb(h_t) \setminus \Rb_t) \ge \eprime$, so 
$\widehat{W}_t(\Qb)$ increases by at least $\eprime$
during this step.  However, since $\widehat{W}_t(\Qb) \le N$, and since
$\widehat{W}_N(\Qb) \ge (1-\eprime) \cdot N$ at the beginning of the second phase
from Corollary~\ref{corr:overall_benefit2}, 
there can be at most $N$ steps in the second phase where
$\widehat{W}_t(\Qb)$ increases by at least $\eprime$.
Thus, during one of the steps in the second phase $\mu(B(h_t)) \le 2 \cdot \eprime = \epsilon$, 
at which point the \basealg\ can select $h_t$ as an $(1-\epsilon)$-good hypothesis and terminate.
This concludes the proof of 
Theorem~\ref{stick-with-it-theorem}.
\end{proof}

\section{Generalization bound}
\label{sec:generalization}

The following generalization bound holds for {\em any} consistent hypothesis at the end of Phase 1. 
\begin{theorem}
With probability at least $1-\delta$, any $h \in \H$ that remains 
in the version space at the end of Phase 1 has $\errcomp(h) < \epsilon$.
\label{thm:generalization-errc}
\end{theorem}
\begin{proof}
Let $\mub$ be the distribution over $\Qb$ that corresponds to picking $q$ from $\mu$
and then picking a feature at random: $\mub(q,j) = \mu(q)/c$. Thus for any $h \in \H$, we have $\mbox{err}_c(h) = \mub(\Bb(h))$.

At the end of Phase 1, $\widehat{W}_N(\Qb) \ge (1-\eprime)\cdot N$. Thus for any $h \in \H$,
\begin{equation*} 
\widehat{W}_N(\Bb(h))
\geq
\left(\sum_{(q,j) \in \Bb(h)} \tau \cdot \mu(q) \right) - \eprime \cdot N
= 
\left(\sum_{(q,j) \in \Bb(h)} N \cdot \mub(q,j) \right) - \eprime\cdot N
=
N \cdot (\mub(\Bb(h)) - \eprime ).
\end{equation*}
If $\mub(\Bb(h)) \geq \epsilon = 2 \cdot\eprime$, we get
\begin{equation*}
W_N(\Bb(h))
\geq
\widehat{W}_N(\Bb(h))
\geq
N\cdot \eprime
> c\cdot \ell
.\end{equation*}
By Lemma~\ref{lemma:generalization}, with probability at least $1-\delta$, any such $h$ is eliminated by the end of the $N$th step.
\end{proof}

Recall from Theorem~\ref{thm:lower-bound-errc} that this $c/\epsilon$ dependence is inevitable.

\section{Lower bound on number of steps and selected hypotheses}
\label{sec lower}

\begin{theorem}
Pick any positive integers $\ell$ and $c$, and any $0 < \epsilon < 1/2$. There exist:
\begin{itemize}
\item a hypothesis class $\H$ of size roughly $c^\ell$ and target concept $h^* \in \H$,
\item a set of queries with $c$ components, and
\item a learner that always chooses a concept in $\H$ consistent with feedback that it has received
\end{itemize}
such that the expected number of queries before the learner arrives at a concept of error $< \epsilon$ is proportional to 
\[  \frac{c \cdot \log \left(|\H|\right) }{\log( c ) \cdot \epsilon} .\]
\end{theorem}

\begin{proof}
Define $\calq$ to be a set of size $\lfloor \ell/(2 \epsilon) \rfloor$, with a subset $\calq^{\epsilon}$ of size $\ell$. The distribution over queries $\calq$ is taken to be uniform. 

Hypothesis class $\H$ consists of binary-valued functions $h$ on $\{(q,j): q \in \Q, j \in [c]\}$ such that:
\begin{itemize}
\item for any $q \not\in \calq^\epsilon$: $h(q,j) = 0$ for all $j \in [c]$
\item for any $q \in \calq^\epsilon$: $h(q,j) = 1$ for at most one component $(q,j)$ 
\end{itemize}
Therefore, $|\H| = (c+1)^{|\calq^{\epsilon}|}$. The target hypothesis $h^* \in \H$ is zero everywhere. 

Let's say the learner always selects as its current hypothesis some $h \in \H$ that is consistent with the feedback it has received, but otherwise disagrees as much as possible with $h^*$ (that is, takes value 1 in as many locations as possible). Then, for each $q \in \calq^{\epsilon}$, this $h$ will take value 1 on some component $(q,j)$ unless $q$ has been queried $c$ times. And unless this occurs for at least half the queries $q \in \calq^\epsilon$, the resulting $h$ will have error $> \epsilon$.

Since a random query is in $\calq^{\epsilon}$ with probability $2\epsilon$, the expected number of queries needed before the learner obtains a hypothesis of error $\leq \epsilon$ is proportional to
$$ \frac{1}{\epsilon} \cdot c \cdot |\calq^{\epsilon}| \approx \frac{c \cdot \log(|\H|)}{\epsilon \cdot \log (c)}. $$
\end{proof}

More generally, the above learner can be modified to use a stick-with-it algorithm, where when a current hypothesis is
selected it is consistent, but it remains the current hypothesis for a number of steps even if it is inconsistent.
Because for each $q \in  \calq^{\epsilon}$ the current hypothesis has value 1 in one component of $q$ if the current hypothesis hasn't been changed at least $c$ times,
$\err(h) \ge 2 \cdot \epsilon$ for the current hypothesis $h$
until the current hypothesis has been changed at least $c$ times. 

\acks{This work is a direct result of the Foundations of Machine Learning program at the  Simons Institute, UC Berkeley.}

\bibliography{partial}

\end{document}